\documentclass[twoside]{article}

\usepackage[accepted]{aistats2022}

\usepackage[utf8]{inputenc} %
\usepackage[T1]{fontenc}    %
\usepackage{url}            %
\usepackage{booktabs}       %
\usepackage{amsfonts}       %
\usepackage{nicefrac}       %
\usepackage{microtype}      %
\usepackage{xcolor}         %
\usepackage{wrapfig}
\usepackage{subcaption}
\usepackage{algorithm}
\usepackage{algorithmicx}
\usepackage{algpseudocode}
\usepackage{amssymb,amsmath}
\usepackage{amsthm}
\usepackage{thmtools, thm-restate}
\usepackage[mathscr]{eucal} %
\usepackage{natbib}
\usepackage{bbm}
\usepackage{arydshln}
\usepackage{enumitem}
\usepackage{cancel}
\usepackage{graphicx}
\usepackage{mdframed}
\usepackage{multicol}
\usepackage{mathtools}

\definecolor{mydarkblue}{rgb}{0,0.08,0.45}

\mdfdefinestyle{mdframedthmbox}{%
	leftmargin=.0\textwidth,
	rightmargin=.0\textwidth,%
	innertopmargin=0.75em,
	innerleftmargin=.5em,
	innerrightmargin=.5em,
}

\usepackage{amsmath,amsfonts,bm}
\usepackage{hyperref}

\newtheorem{definition}{Definition}[section]
\newtheorem{assumption}{Assumption}[section]
\newtheorem{theorem}{Theorem}[section]
\newtheorem{corollary}{Corollary}[theorem]

\newtheorem{lemma}[theorem]{Lemma}

\def\eqref#1{equation~\ref{#1}}

\def\1{\bm{1}}

\usepackage{mathrsfs}

\def\RM{{\mathscr{M}}}

\def\RT{{\mathscr{T}}}

\def\vm{{\bm{m}}}

\def\vx{{\bm{x}}}

\def\mA{{\bm{A}}}
\def\mB{{\bm{B}}}

\def\mI{{\bm{I}}}

\def\mK{{\bm{K}}}

\def\mM{{\bm{M}}}

\def\mQ{{\bm{Q}}}

\def\mV{{\bm{V}}}
\def\mW{{\bm{W}}}
\def\mX{{\bm{X}}}

\def\m0{{\bm{0}}}

\DeclareMathAlphabet{\mathsfit}{\encodingdefault}{\sfdefault}{m}{sl}
\SetMathAlphabet{\mathsfit}{bold}{\encodingdefault}{\sfdefault}{bx}{n}

\def\gF{{\mathcal{F}}}

\def\gM{{\mathcal{M}}}

\def\gT{{\mathcal{T}}}

\def\sP{{\mathbb{P}}}

\newcommand{\E}{\mathbb{E}}

\newcommand{\R}{\mathbb{R}}

\DeclareMathOperator*{\argmax}{arg\,max}
\DeclareMathOperator*{\argmin}{arg\,min}

\DeclareMathOperator{\tr}{tr}

\newcommand{\norm}[1]{\left\lVert #1 \right\rVert}
\newcommand{\lv}{\lVert}
\newcommand{\rv}{\rVert}

\begin{document}

\twocolumn[

\aistatstitle{Meta Learning MDPs with Linear Transition Models}

\aistatsauthor{ Robert M\"uller \And Aldo Pacchiano  }

\aistatsaddress{ Technical University of Munich \And  Microsoft Research, NYC } ]

\begin{abstract}
We study meta-learning in Markov Decision Processes (MDP) with linear transition models in the undiscounted episodic setting.
Under a task sharedness metric based on model proximity we study task families characterized by a distribution over models specified by a bias term and a variance component. We then propose BUC-MatrixRL, a version of the UC-Matrix RL algorithm~\citep{yang2019reinforcement} and show it can meaningfully leverage a set of sampled training tasks to quickly solve a test task sampled from the same task distribution by learning an estimator of the bias parameter of the task distribution. The analysis leverages and extends results in the learning to learn linear regression and linear bandit setting to the more general case of MDP's with linear transition models. We prove that compared to learning the tasks in isolation, BUC-Matrix RL provides significant improvements in the transfer regret for high bias low variance task distributions. 
\end{abstract}

\section{Introduction}
Meta learning~\citep{schmidhuber1987, 287172} is a long-standing quest in machine learning. The goal is to use the experience gained in past tasks to solve new tasks, coming from the same task distribution, quickly. Since meta learning is a problem formulation it can be combined with supervised learning, bandits or reinforcement learning (RL). While the move to bandits adds the challenge of exploration, the move to RL additionally introduces temporal dependence in the data. 

 To better understand which type of transfer is possible in meta RL we ask the following questions: i) what is a family of MDP's? ii) how are the family members related? iii) how can we algorithmically leverage the shared structure? iv)  given a characterisation of a family of MDP's, can we describe and quantify the gain of doing meta RL compared to standard RL on the individual tasks (\textbf{I}ndepedent \textbf{T}ask \textbf{RL})? Recently, there has been a growing body of empirical work in meta RL~\citep{varibad, pearl}, addressing the question of how to algorithmically leverage shared structure. While these successes are impressive, it remains largely open to study conditions under which and to what extent transfer of knowledge is possible. 

 Depending on the setting, the definition of the task distribution is a delicate matter. For general MDP it remains an open question to quantify and formalise task similarity. Recently ~\citep{mueller2020herd} proposed different task embeddings, which are either coupled or independent of the existence of a differentiable policy, that give rise to a distance between tasks. Another distance metric between MDP is the maximum of state-action wise differences in transitions and reward. The well-known simulation lemma~\citep{kearns2002near} uses it to quantify the zero-shot performance of a policy in an unknown MDP which is a known distance apart from a known MDP.   %

 An increasingly well-understood setting for sequential decision-making is to assume linear structure in the MDP, for example in the transition model ~\citep{yang2019reinforcement, modi2019sample, jia2020model, ayoub2020modelbased, cai2020provably, zhou2020provably, he2020logarithmic, zhou2021nearly,Jin2020ProvablyER}. %
 To come closer to study what and when meta RL is useful, we restrict ourselves in a first step to the subset of MDP with a linear transition model. This restriction allows us to represent each MDP uniquely via its transition core matrix; the space of matrix norms equips us with a natural notion of distance between tasks. In this restricted setting, we answer the question:
 
 \emph{Can meta RL yield performance gains compared to ITRL in the space of MDP with linear transition model?} 

To quantify the gain of using meta RL we analyse the transfer regret. Specifically, we propose and analyse in section \ref{sec:biased_ucmatrixRL} a biased version of UC-MatrixRL that follows the spirit of meta learning in ridge regression settings via learning a bias in~\citep{denevi2018learning, denevi2019learning, pmlr-v119-cella20a}. While doing so, we tighten the analysis of UC-MatrixRL \citep{pmlr-v119-cella20a} by using a more careful log-determinant lemma. We quantify the transfer regret in an oracle setting and characterize properties of families of tasks for which meta RL improves compared to independent task learning. Since the oracle setting is not realistic, we conclude our paper by proposing two practical estimators of the mean task distribution transition core and analyse the resulting transfer regret. 

Our paper leverages results from learning to learn in a linear regression setting. The basic concept of learning a bias vector or matrix on the meta train tasks to transfer knowledge via biased linear regression to the test tasks is applicable to all (provably efficient) algorithms based on linear regression in any MDP setting. For ease of presentation and because of its simplicity, we demonstrate the machinery on MDP with linear transition matrices. 

\section{Related Work}
\label{sec:related_work}

While deep meta reinforcement learning has obtained large empirical successes, we  wish to emphasize the lack of theoretical work in meta-learning for RL. A common approach in RL theory is to study tabular problems or MDP under linearity assumptions. This includes in particular linear mixture MDP, whose transition model is a linear function of an a priori given feature map of state-action-next state triples ~\citep{modi2019sample, jia2020model, ayoub2020modelbased, cai2020provably, zhou2020provably, he2020logarithmic, zhou2021nearly}, with the special case of MDP with linear transition model \citep{yang2019reinforcement}.  %
This is the special case considered by our work. 

Inspired from empirical successes, there is a growing body of literature on theoretical works learning to learn in the supervised setting ~\citep{denevi2018learning, denevi2019learning, khodak2019provable, konobeev2021distributiondependent, tripuraneni2021provable}. Several works studied meta learning in a bandit setting ~\citep{azar2013sequential, deshmukh2017multitask, pmlr-v119-cella20a, yang2021impact}. 
Some works study multitask RL~\citep{brunskill2013sample, calandriello2015sparse, lu2021power}. Conditions under which zero-shot RL is possible were recently investigated by ~\citep{malik2021generalizable}. 

A particularly related line of work ~\citep{denevi2018learning, denevi2019learning} propose the idea of learning to learn via learning a bias in supervised learning. The knowledge acquired during the meta train phase is distilled into a bias vector. At meta test time, the learner is  regularised towards this learned bias vector. This inspiration is extended by ~\citep{pmlr-v119-cella20a} to the stochastic bandit setting. 
Our paper is to ~\citep{yang2019reinforcement} as ~\citep{pmlr-v119-cella20a} is to \cite{AbbasiYadkori2011ImprovedAF}. 
\vspace{-0.8em}

\section{Background}
\label{sec:background}
We use small bold letters to denote vectors (e.g. $\vx$) and large bold letters to denote matrices (e.g. $\mX$).  Let $\mA \in R^{d \times d}$ be a positive definite matrix. We denote the Mahalanobis norm of a vector as $\Vert \vx \Vert_{\mA} = \sqrt{\vx^T\mA\vx}$.
Similarly, we denote the matrix version of the Malahonobis norm:
$    \Vert \mA^{1/2} \mX \Vert_F^2 = \tr(\mX^T (\mA^{1/2})^T \mA^{1/2} \mX)= \tr(\mX^T\mA\mX) = \sum_j \Vert \mX_j\Vert_{\mA} = \Vert \mX \Vert_{\mA}^2
$.
Further, we denote the 2-1 matrix norm, which is the sum of the euclidean norms of the matrix columns:
$\Vert \mX\Vert_{2,1} = \sum_j \Vert \mX_j\Vert_2 = \sum_j \langle \mX_j, \mX_j \rangle= \sum_j \sqrt{\sum_i \mX_{i,j}^2} $. We abbreviate the set %
$\{1, \dots, n \}$ as $[n]$. 
We denote by $\sum_{n',h}^{n,H}$ the double sum $\sum_{n' = 1}^n \sum_{h=1}^H$.

In this section, we introduce the single task objective regret as well as the multi-task objective of meta transfer regret. We review the UC-MatrixRL algorithm \citet{yang2019reinforcement} and derive a tighter regret bound using a more careful log determinant lemma. 
\subsection{RL in MDP's with linear transitions}
 We start by reviewing the undiscounted episodic setting. An MDP is given as sextuple: $\gM = (S,A,r,P,H, \mu_0)$ with state-space $S$, action-space $A$, reward function $r: S \times A \rightarrow \left[ 0,1 \right]$, transition probabilities $p: S \times A \times S \rightarrow \left[ 0,1 \right]$, horizon $H$ and starting distribution $\mu_0$. To simplify our presentation we will assume a single starting state $s_0$. Extending our results beyond the setting of a Dirac initial distribution is straightforward. 
The learner is to play for $N$ episodes, each consisting of $H$ steps of interaction with the environment, for a total of $T=NH$ transitions. Throughout our discussion we will use both notations of time interchangeably and note the conversion: $t(n,h) = nH + h$. Thus $(n,h) +1 $ denotes either $(n, h+1)$ if $h < H$ or otherwise the first step in the next episode $(n+1,1)$.

A policy is a mapping $\pi: S \times \left[H \right] \rightarrow A$ that maps each state, stage pair to an action. 
 The value function of a policy $\pi$ at stage $h$ is given as:
    $V^{\pi}_{h}(s) = \E \left[\sum_{t = h}^H r(s_t, \pi_t(s_t)) \right]$ and the action value function as $Q^{\pi}_h(s,a) = r(s,a) + P(\cdot |s,a)V_{h+1}^\pi$.
The optimal policy $\pi^*$ is defined as the policy that maximizes the expected sum of future rewards: $\pi^* = \argmax_{\pi \in \Pi} V^{\pi}(s) \;\; \forall s \in S$. The goal of the learner playing action $a_t$ at time $t$ is to minimize the regret after $T=NH$ steps:
\begin{align}
    R_T(\gM) =  \sum_{n=1}^N \left[V^*(s_0) -  \left(\sum_{h=1}^H r(s_{n,h}, a_{n,h}) \right) \right] \;. 
    \label{eq:regret}
\end{align}
A good learner achieves regret that is sublinear in $T$. Using regret as a performance metric allows us to compare different approaches. 

In this work, we consider each MDP $\gM = (S,A,r,P,H, \mu_0)$ as a task. A family of tasks is thus a finite or infinite set of MDPs.  Within this paper we consider task families with shared $S$, $A$ and $r$, so all the change happening between environments is in the transition probabilities $P$ \footnote{Note that, \citet{yang2019reinforcement} point out, that the assumption of a shared reward function could be lifted by adding an optimistic reward estimation step like in LINUCB \citet{dani2008stochastic}}. We restrict our analysis to the set of MDP with linear transition model. 
\begin{definition}[MDP with linear transition model]
A MDP $\gM$ has a \emph{linear transition model}, if for each $(s_t,a_t) \in S \times A, s_{t+1} \in S$, a priori given feature maps $\phi(s_t,a_t) \in \R^d$ and $\psi(s_{t+1}) \in \R^{d'}$ there exists an unknown matrix $\mM^* \in \R^{d \times d'}$, which is referred to as \emph{transition core}, such that:
\begin{align}
    P(\Tilde{s} \vert s,a) = \phi(s,a)^T\mM^*\psi(\Tilde{s}).
\end{align}
\label{def:linear_transition_model_mdp}
\end{definition}
\vspace{-2em}
Confining ourself to the space of MDP with linear transition model we see that each MDP $\gM^*$ is uniquely characterised by its transition core $\mM$. Let $\gT$ be a distribution of transition cores. A family of linear transition MDP's $\RM$ with shared $ S $, $ A $ and features $\phi$ and $\psi$ can similarly be characterised as a set of transition cores $\mM \sim \gT$ of its members. So, we write interchangeably $\E_{\gM \sim \RM}$ and $\E_{\mM \sim \RT}$.

Our paper considers the meta learning setting. In a first phase the learner interacts with the training tasks from the training task distribution $\RM_{train} \subseteq \RM$. The goal of the meta learner is to reduce in the subsequent test phase the expected meta transfer regret after $T$ steps on the test task distribution $\RM_{test}\subseteq \RM$ :
\begin{align}
        \mathrm{Mtr}_T(\RM_{test}) = \E_{\gM \sim \RM_{test}} Regret(T,\gM)\;. %
        \label{eq:transfer_regret}
\end{align}

\subsection{UC-MatrixRL algorithm}

UC-MatrixRL is an algorithm for MDP with linear transitions proposed by \citet{yang2019reinforcement}. It switches between estimating the unknown transition core $\hat{\mM}_n$ on encountered transitions $(s_t, a_t, s_{t+1})$ and acting greedily with respect to an optimistic estimate of the Q-function build using $\hat{\mM}_n$. We describe both steps below.

The feature maps $\phi_{n,h} = \phi(s_{n,h},a_{n,h}) \in \R^d$ and $\psi_{n,h} = \psi(s_{(n,h)+1}) \in \R^{d'}$ are fixed and a priori given. 
Using the identity matrix $\mI \in \R^{d \times d}$, we define the following matrices:
$\mK_\psi = \sum_{\Tilde{s} \in S} \psi(\Tilde{s}) \psi(\Tilde{s})^T$, $ \mV_n = \sum_{n' \leq n, h \leq H} \phi_{n',h}\phi_{n',h}^T $ and $ \mV^\lambda_n = \lambda \mI +\mV_n$. 
Note that $V_n$ is a symmetric matrix, $\mV^\lambda_0 = \lambda \mI$ and we abbreviate $\mM_{n,H} = \mM_n$. We have furthermore:
\begin{align}
    \E \left[\phi_{n,h}\psi_{n,h}^T\mK_{\psi}^{-1} | s_{n,h},a_{n,h} \right] &= \phi_{n,h} \phi_{n,h}^T \mM^* \;.
\end{align}
$\mM_n$ can be estimated via the ridge regression problem:
\begin{small}
\begin{align}
\argmin_{\mM} \sum_{n',h}^{n,H} \Vert  \psi_{n',h}^T \mK_{\psi}^{-1} - \phi_{n',h}^T \mM\Vert_2^2 + \lambda \Vert \mM \Vert_F^2 \;,
\label{eq:unbiased_ridge_regression_problem}
\end{align}
\end{small}
which has the solution:
\begin{align}
    \mM_n = (\mV_n^\lambda)^{-1}\sum_{n',h}^{n,H}\phi_{n',h}\psi_{n',h}^T \mK_{\psi}^{-1}\;.
\end{align}

UC-MatrixRL constructs optimistic value function estimates based on a matrix ball, $B_n = B_{\beta_n}(\mM_n)$ of radius, $\beta_n$, centred at the current estimate of the transition core $\hat{\mM}_n$. In episode $n$, the optimistic Q-function is recursively constructed as:
\begin{small}
\begin{align*}
    Q_{n, H+1}(s,a) &= 0 \; \forall (s,a) \in S \times A\\
    Q_{n,h}(s,a) &= r(s,a) + \max_{\mM \in B_n}\phi(s,a)^TM\psi^TV_{n, h+1}\; h \in \left[H \right],
\end{align*}
\end{small}
where the value function satisfies $V_{n,h}(s) = \max_a Q_{n,h}(s,a)$. Let $\Psi = \left[\psi(s_1), \dots, \psi(s_{|S|}) \right]^T \in R^{S \times d'}$ be the matrix of concatenated $\psi$ features of all states. 
In order to show regret bounds for the algorithm, we require regularity conditions in feature space. 
\begin{assumption}(Feature regularity) For positive constants $C_M, C_\phi, C_\psi, C_\psi'$ we have:
\vspace{-0.2cm}
\begin{enumerate}
    \item $\Vert \phi(s,a) \Vert_2^2 \leq C_\phi\; \forall (s,a) \in S \times A$,  %
    \item $\Vert \Psi K_\psi^{-1}\Vert_{2, \infty} \leq C_\psi' $
    \item $\Vert \Psi^T v\Vert_2 \leq C_\psi \Vert v \Vert_\infty \; \forall v \in \R^S$ 
    \item $\Vert M^*\Vert_F^2 \leq C_M d$
\end{enumerate}
\label{assumption:regular_features}
\end{assumption}
\vspace{-0.2cm}
These assumptions are almost the same as assumptions $2'$ in \citet{yang2019reinforcement}, only 1. differs in that they assume $\Vert \phi(s,a) \Vert_2^2 \leq C_\phi d$. The set of assumptions leads to a construction of a matrix ball of Frobenius norm.

\subsection{Regret bound of UC-MatrixRL}
~\citet{yang2019reinforcement} show in theorem 2 that under assumption \ref{assumption:regular_features}, with first item modified as described in the previous section, UC-MatrixRL has regret upper bounded in $O\left[C_\psi\sqrt{\Vert \mM^*\Vert_F^2 + C_\psi'^2} \ln(1 + \frac{NHC_\phi}{\lambda} ) \right] dH^2\sqrt{T} $. 
The proof relies on a version of the log-determinant lemma (lemma 8) on $| \phi_{n,h}\|_{V_n^{-1}}$. While the remainder of the paper could be done based on this analysis, we reanalyse this self normalised norm of the summed features, which improves the horizon dependence by a factor of $\sqrt{H}$ and allows emphasising the impact of the bias.%

For notational convenience, we will define: $D \coloneqq 1+\frac{nHC_\phi}{\lambda d} $. In appendix \ref{sec:app:useful_lemmas} we prove the following version of the log determinant lemma. 
\begin{restatable}[Alternative Log Determinant Lemma]{lemma}{logDetLemma}
\label{lemma:upper_bound_doubeling}
The following inequality holds,
\begin{equation*}%
     \sum_{n=1}^N \sum_{h=1}^H \| \phi_{n,h}\|_{V_n^{-1}} \leq \sum_{n=1}^N \sum_{h=1}^H 2\| \phi_{n,h}\|_{V_{n,h}^{-1}} +\frac{C_\phi}{\lambda} d \log\left(D\right) .
\end{equation*}
\end{restatable}
Intuitively this means that the self normalised norm of the sum of feature matrices grows significantly slower than $H$. Despite using an inverse norm on the left hand side defined via the $V_n$ matrices and lacking all the inter-episode outer products, we can show the sum of these inverse norms can be upper bounded by the a constant multiple of the sum of the inverse norms that defined by the inter-episode $V_{n,h}$ matrices plus a term of the form $O(d\log(n))$. This lemma allows us to apply the determinant lemma and save a factor of $\sqrt{H}$ in the regret.   %

Let us define $C_{\phi,\lambda} \coloneqq (4 + C_\phi/\lambda)$. Using our version of the log determinant lemma results in the following regret bound. 
\begin{restatable}[Regret UC-Matrix RL]{theorem}{regretUcMatrixRL}
The regret of the UC-MatrixRL algorithm, under assumptions \ref{assumption:regular_features}, after $T = NH$ steps in MDP $\gM$ is upper bounded as:
\begin{align}
\begin{split}
    R_T(\mM^*) \leq & \left(C_\psi' \sqrt{d'd \log\left(NHD\right)} + \sqrt{\lambda} \Vert \mM^* \Vert_F\right) \\ & 2 C_\psi H  \sqrt{C_{\phi, \lambda}T d \ln \left(D \right)}
    \end{split}
\end{align}
\label{thm:regret_unbiased_matrix_rl_self_normalising_frobenis}
\end{restatable}
    \vspace{-0.60cm}
The theorem is a special case of theorem \ref{thm:regret_biased_matrix_rl_self_normalising_frobenius}, the regret of biased UC-MatrixRL, with bias matrix $\mW=0$. Our bound allows replacing a $\sqrt{2H}$ factor by $\sqrt{C_{\phi,\lambda}} = \sqrt{(4 + C_\phi/\lambda)}$ which is independent of the horizon $H$.

\section{Biased UC-MatrixRL with oracle access}
\label{sec:biased_ucmatrixRL}
In this section we describe and analyse a biased version of the UC-MatrixRL algorithm from \citet{yang2019reinforcement}. We will highlight how the optimal bias improves the regret bound. 

Given a fixed bias matrix $\mW\in \R^{d \times d'}$, the learner faces at the end of the $n$-th episode the ridge regression problem:
\begin{align}
\begin{split}
    \mM_n = \argmin_\mM \sum_{n',h}^{n,H} & \Vert  \psi_{n',h}^T \mK_{\psi}^{-1} - \phi_{n',h}^T \mM\Vert_2^2 \\ + &\lambda \Vert \mM -\mW\Vert_F^2 \;.
    \label{eq:biased_ridge_regression_matrixRL}
\end{split}
\end{align}

We refer to $\mW$ as the bias matrix or bias transition core. Using ridge regression with the change of variables $\mM = \mB+\mW$ and solving for $\mB$ yields for $\mM_n$ the solution:
\begin{align}
\begin{small}
     \mW + \left(\mV_n^\lambda \right)^{-1} \sum_{n',h}^{n,H} \phi_{n,h} \left(\psi_{n',h}^T \mK_{\psi}^{-1}- \phi_{n',h}^T\mW \right).
    \label{eq:solution_biased_uc_matrixRL}
\end{small}
\end{align}
The resulting algorithm is given in algorithm \ref{Alg:BUC-MatrixRL}.

{\centering
\begin{minipage}{.99\linewidth}
\begin{algorithm}[H]
\caption{Within Task Biased Upper-Confidence Matrix RL (BUC-MatrixRL)}\label{Alg:BUC-MatrixRL}
\begin{algorithmic}[0]
    \State \textbf{Input:} MDP$(S,A,P,r,s_0,H)$, features $\phi: S \times A \rightarrow \R^d \; , \psi: S \rightarrow \R^{d'}$, $\lambda>0$,  $\hat{W}_0$, $N$ 
    \State \textbf{Initialize:} $\hat{\mM}_0 = \hat{\mW}_0$, $(\mV_0^\lambda)^{-1} = \frac{1}{\lambda} \mI$ 
    \For{episode $n = 1, \dots, N$:} 
        \State  Build optimistic Q-function using $\mM_n, \beta_n$
        \For{step $h = 1, \dots, H:$} 
            \State select greedy action $a_{n,h}$ using $Q_{n,h}$ %
            \State record next state $s_{n,h+1}$
        \EndFor
        \State update feature matrix $\mV_{n+1} = \mV_n + \sum_{h}^H \phi_{n,h}\phi_{n,h}^T$
        \State recompute ellipsoid radius $\beta_{n+1}$ like in eq. \ref{eq:frobenius_oracle_ellipsoid} 
        \State possibly update bias estimate $\hat{\mW}_{n+1}$ 
        \State recompute core estimate $\hat{\mM}_{n+1}$ using eq. \ref{eq:solution_biased_uc_matrixRL}
    \EndFor
\end{algorithmic}
\end{algorithm}
 \end{minipage}
}

\subsection{Analysis of biased UC-MatrixRL}
We begin by analysing the biased UC-MatrixRL algorithm and analyse afterwards ways to obtain an estimate of the bias. The proof works in two steps, first we show that the optimal transition core $\mM^*$ is with high probability in an over time shrinking ball around $M_n$ and secondly we bound the resulting regret. The employed proof techniques are primarily based on \cite{yang2019reinforcement}, \cite{AbbasiYadkori2011ImprovedAF} and \cite{pmlr-v119-cella20a}.

\subsubsection{Naive construction of confidence balls}
\label{subsec:naive confidence balls}

In order to construct confidence balls we will use self normalizing inequalities. We use them to bound the deviation of the optimal transition core from the algorithm's estimate $\hat{\mM_n}$. 
\begin{restatable}[Confidence Ellipsoid BUC-MatrixRL]{theorem}{ConfidenceSetBucMatrixRL}
Let $\delta>0$. BUC-MatrixRL with bias matrix $\mW$, produces under the assumptions ~\ref{assumption:regular_features} with probability at least $1-\delta$ for all $t >0$ estimates $\hat{\mM_n}$, of the core matrix,  that satisfy:
\begin{align}
\begin{split}
        \Vert\hat{\mM}_n - \mM^* \Vert_F 
        \leq &  C_\psi' \sqrt{2d'\log\left(\frac{1}{\delta}\right) +  d'd \log\left( D\right)} \\+& \sqrt{\lambda} \Vert \mW-\mM^* \Vert_F \eqqcolon \beta_n^\mW(\delta)\;.
    \label{eq:frobenius_noise_ball}
    \end{split}
\end{align}
Thus for all times $t>1$ and recalling the regularity of the feature map $\phi$ we have with probability at least $1-\delta$, that the true transition core $\mM^*$ is contained in the ellipsoid $C_n(\delta) = C^\mW_{\beta_n^\mW(\delta)}(\hat{\mM}_n) = \{\mM \in \R^{d \times d'} | \Vert\hat{\mM}_n - \mM\Vert_F \leq \beta_n^\mW(\delta) \}$ of radius $\beta_n^\mW(\delta)$ and centroid $\hat{\mM}_n$. Clearly it holds $\beta_n^\mW(\delta) \leq \beta_n^\m0(\delta)$, which combined with item 4 of assumption \ref{assumption:regular_features} yields the radius $\beta_n^\m0(\delta)$:
\begin{align}
\begin{split}
               C_\psi' \sqrt{2d'\log\left(\frac{1}{\delta}\right) + d'd \log\left(D\right)} + \sqrt{\lambda} dC_m
             \end{split}
    \label{eq:frobenius_oracle_ellipsoid}
\end{align}
\label{thm:frobenius_oracle_ellipsoid}
\end{restatable}
\vspace{-0.8cm}
It is a common choice to select $\delta = 1/(NH)$. We defer the proof to the appendix \ref{subsec:app:confidence_set_construction}. Having obtained the confidence ball, we are now ready to obtain a regret bound for BUC-MatrixRL. 

\subsubsection{Regret bound}
Using the fact that an ellipsoid of shrinking volume around the current core matrix estimate contains the true core matrix with high probability, results in the following regret:
\begin{restatable}[Regret BUC-Matrix RL]{theorem}{regretBucMatrixRL}
Under the assumptions of theorem \ref{thm:frobenius_oracle_ellipsoid} after $T=NH$ steps, choosing the ellipsoid radius $\beta^\mW_n(\delta)$ as in \ref{thm:frobenius_oracle_ellipsoid} BUC-MatrixRL abides with probability at least $1-1/(NH)$ the following bound on the regret:
\begin{align*}
\begin{split}
        R_T(\mM^*) \leq 
        & \left(C_\psi' \sqrt{d'd \log\left( TD\right)} + \sqrt{\lambda} \Vert \mW-\mM^* \Vert_F\right)\\
        & 2 C_\psi H  \sqrt{C_{\phi,\lambda}Td \ln \left(D \right)}
    \end{split}
\end{align*}
\label{thm:regret_biased_matrix_rl_self_normalising_frobenius}
\end{restatable}
\vspace{-0.5cm}
We defer the proof to appendix \ref{subsec:app:uc_matrixrl_regret}. We wish to emphasise the behaviour of biased UC-MatrixRL for two special choices of bias transition cores, the origin $\m0 \in \R^{d \times d'}$ and the true transition core $\mM^*$:

\begin{enumerate}
    \item Choosing an \textbf{uninformed transition core} $W = \m0 \in \R^{d \times d'}$ recovers unbiased UC-MatrixRL, thus the regret bound from theorem \ref{thm:regret_unbiased_matrix_rl_self_normalising_frobenis} applies. 
    \item \textbf{Oracle Prior} When the algorithm has access to an oracle, it is possible to set the bias transition core of BUC-MatrixRL to the true transition core, $\mW = \mM^*$. As a result the term $\sqrt{\lambda} \Vert \mW-\mM^* \Vert_F$ vanishes.
        Recalling the definition of $D$ as $1+\frac{nHC_\phi}{\lambda d},$ it is clear that the regret goes to $0$ as $\lambda \rightarrow \infty$. Thus, the bigger the regularisation strength, $\lambda$, the smaller the regret we suffer. This is since for large $\lambda$, $\Bar{\mM}_n = \mM^*$, thus the learner selects action greedily with respect to the true Q-function obtained by using the true dynamics. 
\end{enumerate}

\subsection{Meta Transfer Regret Bound}
The goal of meta learning to generalise from the training tasks to all tasks of the task distribution at hand. While we bounded the single task regret in the previous section, what we are really interested in is the transfer regret, thus $E_{\gM \sim \RM}R_T(\gM)$, where $\RM$ is the distribution of all MDP. 
To do so, we define the mean and variance of the task distribution. As each sampled transition core $M$ gives rise to an MDP $\gM \in \RM$, we have an equivalence of sampling transition cores $M$ from the distribution of transition cores $\RT$ or sampling MDP $\gM$ from the distribution of MDPs $\RM$. Thus, we can write interchangeably $\E_{\gM \sim \RM}$ and $\E_{\mM \sim \RT}$. Recall the transfer regret from equation \ref{eq:transfer_regret} as being the expectation of the individual task regret across the distribution of considered MDPs $\RM$. 

We are interested to study the impact of the bias matrix $\mW$ on the incurred transfer regret. Let us define the variance and the mean absolute distance of the bias matrix relative to the distribution of environments $\gT$: 
\begin{align*}
    \text{Var}_W(\gT) = \E_{\mM \sim \gT} \left[ \Vert \mM - \mW \Vert_{F}^2 \right] \\ \text{Mad}_W(\gT) = \E_{\mM \sim \gT} \left[ \Vert \mM - \mW \Vert_{F} \right]\;.
\end{align*}
We skip the argument, the task distribution $\gT$, whenever it is clear form the context that it is either the training task or test task distribution. We are now ready to state a first bound on the $\mathrm{Mtr}$:
\begin{restatable}[Meta Transfer Regret BUC-MatrixRL]{theorem}{MetaRegretBucMatrixRL}Under the assumptions %
of theorem \ref{thm:regret_biased_matrix_rl_self_normalising_frobenius} we have with probability at least $1-1/(NH)$ for a task distribution $\gT$ the following $\mathrm{Mtr}$ after $T$ steps per task (where we absorb constant factors into $C$):
\begin{align*}
\begin{split}
        \mathrm{Mtr}_T(\gT)  \leq &
        C C_\psi H C_\psi'd\sqrt{d'TC_{\phi, \lambda}\log\left( TD\right)\ln \left(D \right)}\\ & + CC_\psi H \textup{Mad}_\mW \sqrt{\lambda T C_{\phi, \lambda} d \ln \left(D \right)}\\
     \leq &
    C C_\psi H C_\psi'd\sqrt{d'TC_{\phi, \lambda}\log\left( TD\right)\ln \left(D \right)}\\ & + CC_\psi H \sqrt{\textup{Var}_\mW\lambda T C_{\phi, \lambda}d \ln \left(D \right)}
\end{split}
\end{align*}\label{thm:transfer_regret_frobenius_oracle}
\end{restatable}
\vspace{-0.5cm}
We will now interpret the obtained upper bound on the transfer regret for different choices of regularisation strength $\lambda$ and bias matrix $\mW$: %
\begin{enumerate}
    \item We obtain in the limit of an infinite regularisation strength (omitting a constant 4 from $C_{\phi, \lambda}$):
    \begin{align*}
    \begin{split}
        &\lim_{\lambda \rightarrow \infty} \mathrm{Mtr}_T(\gT) \\
            &\leq \lim_{\lambda \rightarrow \infty}CC_\psi H \sqrt{\textup{Var}_\mW\lambda T d \ln \left( 1+\frac{nHC_\phi}{\lambda d} \right)}\\
         &= CC_\psi H \sqrt{\textup{Var}_\mW T^2C_\phi}
    \end{split}
    \end{align*}
    a $\sqrt{\textup{Var}_\mW}$ dependence of the transfer regret on the variance of the task distribution with respect to the bias matrix $\mW$.  
    \item It turns out that for a particular choice of $\lambda$ we can actually improve the $\sqrt{\textup{Var}_\mW}$ dependence of the transfer regret. For a matrix $\mQ \in \R^{d \times d'}$, we define the function $\text{DVar}: R^{d \times d'} \rightarrow \R\;; \text{DVar}(Q) \coloneqq \left(1 + \frac{T^2 \text{Var}_{\mQ} C_\phi}{d}\right)$. Let $\lambda = \frac{1}{T\textup{Var}_W}$, which is motivated by the following two aspects. Firstly, the reliance on the prior should shrink the more samples are seen in a particular task, thus the $1/T$ part. Moreover, we want the learner to rely less on the prior $\mW$ if the task distribution is very broad as indicated by a high task variance $\textup{Var}_\mW$, thus the $1/\textup{Var}_\mW$. The resulting transfer regret is:
    \begin{align*}
    \begin{split}
        \mathrm{Mtr}_T(\gT) & =\left( 1 +  C_\psi' \sqrt{d'dT \log\left( T \text{DVar}(\mW)\right)  } \right)\\
        &CC_\psi H\sqrt{ C_{\phi, \lambda}d \ln \left(\text{DVar}(\mW) \right)}   \;.
        \end{split}
        \end{align*} 
    This means that the choice of regularisation strength $\frac{1}{T\textup{Var}_\mW}$ yields a $\sqrt{\log(1+\textup{Var}_\mW)}$ dependence, whereas the limit of $\lambda \rightarrow \infty$ yields a worse dependence of $\sqrt{\textup{Var}_\mW}$. 
    \item Keeping $\lambda = \frac{1}{T\textup{Var}_\mW}$ the transfer regret BUC-MatrixRL with the bias matrix chosen as mean transition core $\Bar{\mM}$ goes to zero as the variance of the task distribution goes to 0. 
    \item \textbf{Oracle BUC-MatrixRL} improves against individual task learning, whenever the variance of the task distribution is much lower than its offset from the origin:
\begin{align*}
\begin{split}
        &\textup{Var}_{\Bar{\mM}} = \E_{\mW \sim \gT} \Vert \mM - \Bar{\mM} \Vert_F^2 \\&\ll \E_{\mM \sim \gT} \Vert \mM \Vert_F^2 = \textup{Var}_0 \;.
        \end{split}
\end{align*}
\end{enumerate}

Choosing an appropriate step size of $\lambda = \frac{1}{T\textup{Var}_\mW}$ allows for a dependence of the upper bound of the transfer regret of $\sqrt{\ln(1+\textup{Var}_\mW)}$.
In the limit of no variance of the task distribution we recover the single task regret bound from theorem \ref{thm:regret_biased_matrix_rl_self_normalising_frobenius}.

\section{Practical Meta Learning with biased UC-Matrix RL}
So far we assumed access to an oracle of the optimal transition core $\mM^*$ and showed its usefulness in terms of incurred transfer regret. In any practical setting we do however have no access to an oracle, thus we want to transfer knowledge from previous tasks. Following the meta learning protocol, we have access to $G$ many training tasks. Each task $g$ was created by sampling $M_g \sim \gT$. We run biased UC-MatrixRL with a bias transition core based on the transition core estimates of previous tasks for $N$ episodes a $H$ steps in MDP $\gM_g$. To enable generalisation across the full distribution of tasks, we use at meta test time  BUC-Matrix RL with a bias distilling the transferable knowledge from meta training. The full training protocol %
is given in algorithm \ref{Alg:Meta-Train}. 
In the case of sequentially arriving tasks $1, \dots, g$ we can always consider all previous tasks $1, \dots, g-1$ as training tasks, and the test set as the one element set containing only the $g$-th task. Thus, our approach caters also to the continual learning setting. 
{\centering
\begin{minipage}{.99\linewidth}
    \begin{algorithm}[H]
        \caption{Meta Train}
    \label{Alg:Meta-Train}
    \begin{algorithmic}[0]
    \State \textbf{Input:} set of training tasks $\gM_1, \dots, \gM_G$, features $\phi: S \times A \rightarrow \R^d \; , \psi: S \rightarrow \R^{d'}$, $\lambda>0$,  bias for the first task $\hat{\mW}_0$, episode number $N$, horizon $H$ 
    \State \textbf{Initialize:} $\hat{\mM}_{0,0} = \hat{\mW}_0$, $\mV_0^{-1} = \frac{1}{\lambda} \mI$ 
    \For{train task $\gM_g \in \{ \gM_1, \dots, \gM_G \}$:}
    \State run alg. \ref{Alg:BUC-MatrixRL} on MDP $\gM_g$ with bias $\mW = \hat{\mW}_{g-1}$
    \EndFor
    \end{algorithmic}
    \end{algorithm}
\end{minipage}
}

Inspired from the bias estimators in \cite{pmlr-v119-cella20a} we describe now two methods to estimate the bias transition core matrix in episode $n$ of task $g$, $\mW_{g,n}$, based on previous experience.

To account for the multitask setting we add a subscript indicating the task index $g \in \left[G \right]$. Recall our assumption that all tasks share the same feature maps $\phi$ and $\psi$ and the interchangeability of the $T$ and $nh$ time within a task. For the $g$-th MDP, we have the transition core $M_g$, feature matrix $\mV_{g,t} = \mV_{g,t//H, t\%H} = \sum_{p \leq t} \phi(s_{g,p}) \phi(s_{g,p})^T =\sum_{p \leq t} \phi_{g,p}\phi_{g,p}^T$, the concatenation of all features in episode $n$: $\Phi_{g,n} = \left[\phi_{g,n,1}, \dots, \phi_{g,n,H} \right]^T$ and $\Psi_{g,n} = \left[\psi_{g,n,1}, \dots, \psi_{g,n,H} \right]^T$ and the radius of the confidence ellipsoid at stage $t$ as $\beta^\m0_{g,t}$. A quantity central to our analysis is the mean estimation error of the bias $\hat{\mW}_{g,n,h}$ at stage $(n,h)$ of the $g$-th task with respect to the true mean transition core $\Bar{\mM}$:
\begin{align}
    \epsilon_{g,n,h} (\RM) = \Vert \Bar{\mM} - \hat{\mW}_{g,n,h}\Vert_F^2 \;.
    \label{eq:mean_estimation_error_bias}
\end{align}
As the learner faces the tasks in a sequential manner there is an inherent estimation error, due to the fact that it has only samples from the task distribution. We define $\Bar{W}_{G,t} = \frac{1}{GT + t} \left(\sum_{g=1}^G T \mM_g + tM_{G+1} \right)$ the mean transition core of the observed MDP and denote the estimation error relative to the true mean transition core $\Bar{\mM}$ as $H_{\RM}(G+1,\Bar{\mM}) = \Vert \Bar{\mM} - \Bar{\mW}_{G,t} \Vert_F$.

\subsection{Averaging previous transition core estimates - a low bias estimator }
\label{subsection:buc-matrixrl-low_bias_estimator}
The first approach is to use an weighted average of previous transition core estimates as bias. The motivation is that the knowledge acquired in previous MDP $\gM_1$, \dots, $\gM_G$ is distilled in the respective final estimate of the transition core $\hat{\mW}_{G,T}$. Knowledge transfer between the tasks is achieved by aggregation of the individual estimated transition cores:
\begin{align}
\begin{split}
    \hat{\mW}_{G,n,h}  =  \sum_{g=1}^{G-1} \frac{T}{Z}\hat{\mM}_{g,T}  + \frac{nH+h}{Z} \hat{\mM}_{G,n,h},
    \label{eq:bias_estimator_weighted_average}
    \end{split}
\end{align}
with normalisation factor $Z = T(G-1) + nH + h$. 
This choice of bias estimator results in the following bound on the transfer regret. 
\begin{restatable}[]{theorem}{LowBiasMetaRegretBucMatrixRL}
BUC-MatrixRL incurs after $T$ interactions in $G$ previous tasks, using the bias estimator $\hat{\mW}_{G,n,h}$ as in equation \ref{eq:bias_estimator_weighted_average}, $\lambda = \frac{1}{T\text{Var}_{\hat{\mW}_{G,n,h}}}$ and under assumptions \ref{assumption:regular_features} with probability at least $1-1/(NH)$, at most the following meta transfer regret:
\begin{align*}
\begin{split}
    &\mathrm{Mtr}_T(\gM_{G+1}) \leq C C_\psi HdC_\psi' \\ &\sqrt{C_{\phi, \lambda}d'T \log\left( T + \frac{T^3C_\phi \left(Var_{\Bar{M}} + \epsilon_{G,T}(\RM)\right)}{d}\right)  }
\end{split}
\end{align*}
The mean estimation error can be upper bounded as:
\begin{align*}
    \sqrt{\epsilon_{G,T}(\RM)} \leq H_{\RM}(G+1, \Bar{\mM}) + \max_{g \in \left[G \right]} \frac{\beta^\m0_{g,T}(1/{NH})}{\lambda_{min}^{1/2} (V^\lambda_{g,T})}
\end{align*}
\label{thm:transfer_regret_low_bias_bias_estimator}
\end{restatable} %

\vspace{-1em}
The proof is based on theorem \ref{thm:transfer_regret_frobenius_oracle} and deferred to appendix \ref{sec:app:low_bias_transfer_regret_proof}. 
Analysing the bound on the noise we show that, by a Frobenius matrix norm version of Bennetts inequality, the first term goes for an increasing number of tasks to zero. This means the estimation error is dominated by the second term. Recall our choice of $\lambda = 1/(TVar_{\hat{\mW}})$, we see that the estimation error increases with the variance of our bias matrix estimator. The meta learning procedure comes with an additional storage requirement of the size of the transition core in which the estimates of the training tasks are averaged. 

\subsection{Global ridge regression - a high bias estimator}
\label{subsection:buc-matrixrl-high_bias_estimator}
The previous estimator shared knowledge between MDP's $\gM_g$ via the final estimated transition cores. Here we present an estimator that instead builds features on all transitions seen in all previous MDP and performs one global ridge regression to estimate the transition core matrix which we use as bias in biased UC-MatrixRL. This approach is inspired by the high bias estimator in \citep{pmlr-v119-cella20a} and is in line with previous estimators in the multitask bandit literature. The knowledge transfer between tasks works thus in form of feature embeddings of the observed $(s,a,s')$ transitions instead of aggregated objects. Let $\Tilde{\mV}_{G,n,h} \coloneqq \sum_{g=1}^G\mV_{G,T} + \mV_{G+1,n,h}$. Global ridge regression uses the estimator:
\vspace{-0.1cm}
\begin{align}
\begin{split}
    \hat{\mW}_{G,n,h} = & (\Tilde{\mV}_{G,n,h}^\lambda)^{-1} \Bigg[\sum_{g=1}^{G-1} \sum_{n,h}^{N,H} \phi_{g,n,h}\psi_{g,n,h}\mK_\psi^{-1} \\ &+ \sum_{n',h'}^{n,h}  \phi_{G,n',h'}\psi_{G,n',h'}\mK_\psi^{-1}\Bigg]\;.
     \vspace{-0.5cm}
    \label{eq:bias_estimator_global_regression}
    \end{split}
\end{align}
This yields the following meta transfer regret: \begin{restatable}[]{theorem}{HighBiasMetaRegretBucMatrixRL} BUC-MatrixRL incurs after $T$ interactions in $G$ previous tasks, using the bias estimator $\hat{\mW}_{G,n,h}$ as in equation \ref{eq:bias_estimator_global_regression}, $\lambda = \frac{1}{T\text{Var}_{\hat{\mW}_{G,n,h}}}$ and assumptions \ref{assumption:regular_features} with probability at least $1-1/(NH)$, at most the following meta transfer regret:
\begin{align*}
\begin{split}
        &\mathrm{Mtr}_T(\gM_{G+1}) \leq C C_\psi HdC_\psi' \\& \sqrt{C_{\phi, \lambda}d'T \log\left( T + \frac{T^3C_\phi \left(Var_{\Bar{M}} + \epsilon_{G,T}(\gT)\right)}{d}\right)  }
        \end{split}
\end{align*}
    \vspace{-0.5cm}

Let $\nu_{\min} = \lambda_{\min}\left( \Tilde{\mV}_{G,n,h}\right)$ be the minimal singular value of the global feature matrix. Then the mean estimation error can be upper bounded as:
\begin{align*}
\begin{split}
    &\sqrt{\epsilon_{G,T}(\RM)} 
    \leq H_{\RM}(G+1, \Bar{\mM})  \\
        & + \underbrace{\frac{d C_M}{\lambda + \nu_{\min}} + C_\psi' \sqrt{ \frac{2}{\lambda + \nu_{\min}}\log\left(NH + \frac{GN^2H^2C_\phi}{\lambda d}\right)}}_{\frac{\beta^\m0(1/(GNH))}{\lambda + \nu_{\min}}}\\
   & + 2(G+1)\max_{g \in \left[G+1 \right]} \Tilde{H}(G+1, \mM_g)
\;.
    \end{split}
\end{align*}
$\Tilde{H}(G+1, \mM_g)$ is a weighted version of the estimation error $H_{\RM}(G+1, \Bar{\mM})$ towards the current transition core $\mM_g$:
\begin{align}
    \Tilde{H}(G,\mM_g) = H_{\RM}(g, \mM_g) \sigma_{\max}\left(\mV_{g,T}\Tilde{\mV}_{G,N,H}^{-1} \right) \;,
\end{align}
where $\sigma_{\max}\left(\mV_{g,T}\Tilde{\mV}_{G,N,H}^{-1} \right)$ quantifies the misalignment of task $g$ to the other tasks observed so far. 
\label{thm:transfer_regret_high_bias_bias_estimator}
\end{restatable}

The regret bound follows from theorem \ref{thm:transfer_regret_frobenius_oracle} and is derived in appendix \ref{sec:app:high_bias_transfer_regret_proof}. Comparing this to the low bias solution in theorem \ref{thm:transfer_regret_low_bias_bias_estimator} we see that the variance is now $\frac{\beta^\m0(1/(GNH))}{\lambda + \nu_{\min}}$ instead of $ \frac{\beta^\m0_{g,T}(1/{NH})}{\lambda_{min}^{1/2}( V^\lambda_{g,T})}$. From observing, $\nu_{min} \geq \frac{G}{dd'}\lambda_{min}(\mV_g)\; \forall g \in \left[ G\right]$ it follows that we shrink the variance by a factor $(dd')/G$. As the number of training tasks $G$ goes to infinity, the variance goes to zero. This comes however at the price of increased bias  $2(G+1)\max_{g \in \left[G+1 \right]} \Tilde{H}(G+1, \mM_g)$ which increases proportional to the task misalignment, $\sigma_{\max}\left(V_{g,T}\Tilde{V}_{G,N,H}^{-1} \right)$. We illustrate its behaviour on two corner cases. 

First, assume that the task distribution $\gT = \{\mM\}$, so our task distribution, consists of a transition single core matrix. As we face always the same task, we expect extremely favourable transfer. In particular, we suffer in this case no bias as $\Tilde{H}(G+1, M_g)  = 0$. 

The other extreme are completely unrelated tasks. The minimal relatedness we can generate is if corresponding vectors of the transition cores are orthogonal to each other. Consider the task distribution, where each column of the transition cores is a $d$-dimensional basis vector. In this case one can do meta learning on $d-1$ mutually different training tasks without seeing an improved meta transfer regret on the $d$-th task. 

Note that the linear MDP models requires feature maps $\phi$ and $\psi$ as input. The transition cores live in a feature map dependent space, making it hard to connect properties of the transition cores of a task family to interpretable properties of the environments. 

\subsection{ITRL vs MTRL}

We can now compare the gain of doing meta learning compared to independent task learning. In the independent learner scenario, the transfer regret depends on the variance of the tasks relative to the origin, $\textup{Var}_{\m0}$. Doing meta learning allows us to obtain estimators that replace the variance with respect to the origin by the variance to the mean task $\textup{Var}_{\Bar{\mM}}$ plus an additional error term. In the oracle case, the error term is zero. Since in practice there is no oracle, we bound the additional error term for two different bias estimation methods. For any task distribution of small variance but large offset, the obtained transfer regret incurred by meta learning is lower than from ITRL.

\section{Discussion}

Our paper gives an affirmative answer to the initial question of the usefulness of meta RL. Using the one-to-one correspondence of MDP with linear transition core and the core matrix, we have a notion of distance between tasks. This allows the characterisation of any distribution of linear transition core MDPs via its offset and variance. We prove a decrease in transfer regret of meta RL compared to independent task learning whenever the variance of the task distribution is small compared to the offset from the origin. 
While we show this improvement first in a setting with access to an oracle that reveals the offset of the transition core distribution, we extend the result to two practical estimators. 

One estimator performs knowledge transfer between the tasks using an aggregation of distilled knowledge obtained on previous tasks in the form of estimated transition cores. This method suffers, however, a possibly large error due to the direct proportionality of the estimation error, thus transfer regret with the variance of the transition core estimator.  The second proposed estimator suffers bias proportional to the task misalignment in a trade-off for a variance that goes to zero as the number of tasks goes to infinity. Here, the transfer of knowledge happens via the sharing of embeddings of the observed transitions. Note that similar meta learning via learning a bias estimators can be combined with any (provably efficient) RL algorithm based on linear regression for similar gains compared to ITRL.

Our work chooses meta transfer regret at performance measure, which weights all test tasks equally. An interesting avenue of future work is to analyse alternative objectives, for example the worst case transfer regret within the test distribution.

A major limitation of the framework of linear transition core MDP is the assumption of known feature maps $\phi$ and $\psi$. This allows studying the usefulness of meta learning for rapid learning in a newly encountered task. In empirically successful meta deep RL works, the feature embeddings are, however, not given but instead learned. Thus, the effectiveness of meta learning could lie within learning a good initialisation/bias or in learning a set of reusable features. For the case of model agnostic meta learning~\citep{finn2017model} the empirical study~\citep{Raghu2019RapidLO} finds that feature reuse is the dominant factor in the examined few shot classification and RL tasks. To fully understand the usefulness of meta learning in general MDP's, it is thus necessary to also take feature learning into account.

We lastly wish to emphasize that each linear MDP can be trivially embedded as its generating vector/matrix. The space of vector-/ matrix norms gives immediately rise to distances between tasks. It remains however an open question how to embed general MDP and to develop distance metrics between the task embeddings that are meaningful, for example in the sense that learning in nearby tasks aids learning in target tasks.  
\newpage
\bibliographystyle{apalike}
\bibliography{references}
\clearpage
\onecolumn

\appendix
\newcommand{\appendixTitle}{%
\vbox{
    \centering
	\hrule height 4pt
	\vskip 0.2in
	{\LARGE \bf SUPPLEMENTARY MATERIAL}
	\vskip 0.2in
	\hrule height 1pt 
}}

\appendixTitle

\section*{Organization of Supplementary Material}
\begin{itemize}
   
    \item[\ref{sec:app:notation}] \nameref{sec:app:notation}
     
    \item[\ref{sec:app:useful_lemmas}] \nameref{sec:app:useful_lemmas}

    \item[\ref{sec:app:analysis_regret_matrixRL}] \nameref{sec:app:analysis_regret_matrixRL}
    \begin{itemize}      \item[\ref{subsec:app:confidence_set_construction}] \nameref{subsec:app:confidence_set_construction}
                \item[\ref{subsec:app:uc_matrixrl_regret}] \nameref{subsec:app:uc_matrixrl_regret}
                \item[\ref{subsec:app:meta_transfer_regret}] \nameref{subsec:app:meta_transfer_regret}
    \end{itemize}

    \item[\ref{sec:app:low_bias_transfer_regret_proof}] \nameref{sec:app:low_bias_transfer_regret_proof}

    \item[\ref{sec:app:high_bias_transfer_regret_proof}] \nameref{sec:app:high_bias_transfer_regret_proof}

\end{itemize}

\section{Notation}
\label{sec:app:notation}
Let $\vx \in \R^d$ and $\mA \in R^{d \times d}$ a positive definite matrix. We define the Mahalanobis norm $\Vert \vx \Vert_{\mA} = \sqrt{\vx^T\mA\vx}$.
For a matrix $\mX \in \R^{d \times d'}$ we have the Frobenius norm and the inducing matrix inner product norm $\Vert \mX \Vert_F = \sqrt{tr (\mX^T \mX)} = \sqrt{ \langle \mX,\mX \rangle_F} = \sqrt{\sum_{i,j} |\mX_{i,j} |^2}$. We have further the Mahalanobis version of the Frobenius norm for a symmetric positive definite $\mA \in R^{d \times d}$:
\begin{align}
    \Vert \mA^{1/2} \mX \Vert_F^2 = \tr(\mX^T (\mA^{1/2})^T \mA^{1/2} \mX)= \tr(\mX^T\mA\mX) = \sum_j \Vert \mX_j\Vert_{\mA} = \Vert \mX \Vert_{\mA}^2
    \label{eq:matrix_malahonis_frobenius}
\end{align}

We denote the column indices of $\mX$ as $j \in \{1, \dots, d' \}$ and row indices $i \in \{1, \dots, d \}$. We denote the 2-1 matrix norm, which is the sum of the euclidean norms of the matrix columns:
$\Vert \mX\Vert_{2,1} = \sum_j \Vert \mX_j\Vert_2 = \sum_j \langle \mX_j, \mX_j \rangle= \sum_j \sqrt{\sum_i \mX_{i,j}^2} $. Similar we can define the $\mA-1$ norm, which is the sum of the Mahalanobis norms of the individual columns: 
$\Vert \mX \Vert_{\mA,1} = \sum_j \Vert \mX_j\Vert_{\mA} = \sum_j \sqrt{ \mX_j^T \mA \mX_j }$. We have further:
\begin{align*}
    \Vert \mA^{1/2}\mX\Vert_{2,1} =
    \sum_j \Vert \mA^{1/2}\mX_j\Vert_2 = \sum_j \sqrt{\mX_j^T  \mA \mX_j} = \sum_j \Vert \mX_j \Vert_{\mA} = \Vert \mX \Vert_{\mA,1}\;.
\end{align*}
We will also use the $2-\infty$ norm of a matrix $\Vert \mX \Vert_{2, \infty} = \max_j \Vert \mX_j \Vert_2$

\section{Useful lemmas}
\label{sec:app:useful_lemmas}
Throughout the proof we will need at different locations the elliptical potential lemma:

\begin{lemma}[Lemma 19.4 in~\citep{lattimore_szepesvari_2020}]
Let $\mV_0 \in \R^{d,d}$ positive definitive and $\phi_1,\dots, \phi_h \in \R^d$ a sequence of vectors with $\Vert \phi_t \Vert_2^2 \leq C_\phi^2 \;\; \forall t\in \left[h\right]$, $\mV_h = V + \sum_{t \leq h} \phi_t^T \phi_t$. Then:

\begin{align}
    \sum_{t=1}^h \min(1, \Vert \phi_t \Vert_{V_t^{-1}}^2) 
    \leq 2 \log \left(\frac{\det \mV_h}{\det \mV_0} \right)
    \leq 2d \log\left(\frac{tr \mV_0 + h L^2}{d \det \mV_0^{1/d}} \right)
\end{align}
\label{app:lemma:eppliptical_potential_bandit}
\end{lemma}

Recalling our feature regularity assumption $\Vert \phi(s,a) \Vert_2^2 \leq C_\phi$ we get:
\begin{align}
    2 \log \left(\frac{\det \mV_h}{\det \mV_0} \right) \leq 
    2d \log\left(\frac{\lambda d + h C_\phi}{d \lambda} \right) = 2d\log\left(1 + \frac{HC_\phi}{\lambda d} \right) \;.
    \label{eq:our_elliptical_potential_lemma}
\end{align}

Note that ~\citep{yang2019reinforcement}(lemma 10) would get here by virtue of choosing $\Vert \phi(s,a) \Vert_2^2 \leq dC_\phi$ and $\lambda = 1$:
\begin{align}
    2 \log \left(\frac{\det \mV_h}{\det \mV_0} \right) \leq 
    2d \log\left(\frac{\lambda d + hd C_\phi}{d \lambda} \right) = 2d\log\left(1 + hC_\phi \right) \;.
\end{align}

\begin{lemma}\label{lemma::supporting_lin_alg_result}
If $\mathbf{B} \succeq \mathbf{C} \succ \mathbf{0}$ be $d\times d$ dimensional matrices then,
\begin{equation*}
    \sup_{\mathbf{x}\neq 0}\frac{\mathbf{x}^\top \mathbf{B} \mathbf{x} }{ \mathbf{x}^\top \mathbf{C} \mathbf{x} } \leq \frac{\mathrm{det}( \mathbf{B}) }{\mathrm{det}( \mathbf{C})}.
\end{equation*}
\end{lemma}
\begin{proof}
Given any $\mathbf{y} \in \mathbb{R}^d$ let $\mathbf{x} = \mathbf{C}^{-1/2}\mathbf{y}$. Then
\begin{align*}
    \sup_{\mathbf{x} \neq 0} \frac{\mathbf{x}^\top \mathbf{B} \mathbf{x}}{\mathbf{x}^\top \mathbf{C} \mathbf{x}}=    \sup_{\mathbf{y} \neq 0} \frac{\mathbf{y}^\top \mathbf{C}^{-1/2}\mathbf{B}\mathbf{C}^{-1/2} \mathbf{y}}{\lv \mathbf{y} \rv_2^2} = \norm{\mathbf{C}^{-1/2} \mathbf{B} \mathbf{C}^{-1/2}}_{op}
\end{align*}
by the definition of the operator norm.  Recall that by assumption $\mathbf{B}-\mathbf{C}\succeq 0$ therefore $\mathbf{C}^{-1/2}\mathbf{B}\mathbf{C}^{-1/2}-\mathbf{I}\succeq 0$, and hence all the eigenvalues of $\mathbf{C}^{-1/2} \mathbf{B} \mathbf{C}^{-1/2}$ are at least $1$. Thus 
\begin{align*}
     \sup_{\mathbf{x} \neq 0} \frac{\mathbf{x}^\top \mathbf{B} \mathbf{x}}{\mathbf{x}^\top \mathbf{C} \mathbf{x}} \le \norm{\mathbf{C}^{-1/2} \mathbf{B} \mathbf{C}^{-1/2}}_{op}\le \det(\mathbf{C}^{-1/2} \mathbf{B} \mathbf{C}^{-1/2}) = \frac{\mathrm{det}( \mathbf{B}) }{\mathrm{det}( \mathbf{C})},
\end{align*}
where the last equality follows since $\frac{\det(\mathbf{B})}{\det(\mathbf{C})} = \det(\mathbf{C}^{-1/2})\det(\mathbf{B})\det(\mathbf{C}^{-1/2})= \det(\mathbf{C}^{-1/2} \mathbf{B} \mathbf{C}^{-1/2})$. This completes the proof.
\end{proof}

Recall that $V_{n,h} = V_n + \sum_{h'<h} \phi_{n,h'}\phi_{n,h'}^\top$. We are now ready to proof the doubling log-determinant lemma.

\logDetLemma*
\begin{proof}[Proof of Lemma \ref{lemma:upper_bound_doubeling}]
 Define $e_{n,h} = \mathbf{1}\left(  \| \phi_{n,h} \|_{(\mV_{n,h}^\lambda)^{-1}}   \leq 2\| \phi_{n,h}\|_{V^\lambda)^{-1}_n} \right) $. We define $e_{n,h}^c =1-e_{n,h}$. 

\begin{equation*}
    \sum_{n=1}^N \sum_{h=1}^H \| \phi_{n,h}\|_{(\mV_n^\lambda)^{-1}} \leq \sum_{n=1}^N \sum_{h=1}^H 2\| \phi_{n,h}\|_{(\mV_{n,h}^\lambda)^{-1}} +   \sum_{n=1}^N \sum_{h=1}^H e^c_{n,h} \frac{C_\phi}{\lambda}.
\end{equation*}

If $e_{n,h}^c =1$ and as a consequence of Lemma~\ref{lemma::supporting_lin_alg_result},

\begin{equation*}
2\leq    \frac{\| \phi_{n,h} \|_{(\mV^\lambda)^{-1}_{n,h}}}{\| \phi_{n,h} \|_{(\mV^\lambda)^{-1}_{n}}} \leq \frac{\mathrm{det} (\mV_{n,h}) }{\mathrm{det}(\mV_n)}.
\end{equation*}

And therefore it must be that,
\begin{equation*}
     2^{\sum_{n=1}^{N}\sum_{h=1}^H e_{n,h}^c } \leq  \frac{\mathrm{det}(\mV_N ) }{\mathrm{det}(\lambda \mathbf{I})}
\end{equation*}

Furthermore, as a consequence of Lemma~\ref{app:lemma:eppliptical_potential_bandit},

\begin{equation*}
   \log\left(  \frac{\mathrm{det} (\mV_N ) }{\mathrm{det}(\lambda \mathbf{I})} \right) \leq  d \log\left( 1+\frac{  n H C_\phi^2 }{\lambda d}\right) . 
\end{equation*}

We therefore conclude that,
\begin{equation*}
    \sum_{n=1}^{N}\sum_{h=1}^H e_{n,h}^c \leq d \log\left( 1+\frac{  H n C_\phi^2 }{\lambda d}\right) . 
\end{equation*}
The result follows.
\end{proof}

Intuitively this means that the self normalised norm of the sum of feature matrices grows significantly slower than $H$. Despite using an inverse norm on the left hand side defined via the $V_n$ matrices and lacking all the inter-episode outer products, we can show the sum of these inverse norms can be upper bounded by the a constant multiple of the sum of the inverse norms that defined by the inter-episode $V_{n,h}$ matrices plus a term of the form $O(d\log(n))$.

\section{Analysing the regret of (biased) UC-MatrixRL}
\label{sec:app:analysis_regret_matrixRL}
The analysis of UC-MatrixRL works in two steps. First we show that the optimal transition core $\mM^*$ is with high probablility contained in an ellipsoid around the current estimate $\mM_n$. Subsequently we proof the regret for the case that $\mM^*$ is with high probability in the constructed confidence ellipsoid. 

We give the proof for the general biased case. The unbiased case follows immediately by choosing the zero matrix as bias, $W = 0$. 

\subsection{Constructing Confidence Sets}
\label{subsec:app:confidence_set_construction}

\ConfidenceSetBucMatrixRL*

\begin{proof}

The proof leverages the confidence ellipsoid from theorem 2 in ~\citep{AbbasiYadkori2011ImprovedAF}.

We start with the solution of biased UC-Matrix RL and write out the regression target to illustrate the impact of the noise:
\begin{align}
    \mM_n & = \left(\mV_n^{\lambda}\right)^{-1} \left[\sum_{n' < n, h \leq H} \phi_{n,h} \left(\psi_{n,h} K_{\psi}^{-1}- \phi_{n,h}^TW \right)\right] +\mW \\
        & = \left(\mV_n^{\lambda} \right)^{-1} \left[\sum_{n' < n, h \leq H} \phi_{n,h} \left( \phi_{n,h}\mM^* + \eta_{n,h}- \phi_{n,h}^TW \right)\right] +\mW \;.
\end{align}
We denote by $\eta_{n,h} = K_\psi^{-1}\psi_{n,h} - (\mM^*)^T\phi_{n,h}\in R^{d'}$ the noise vector. Using assumption \ref{assumption:regular_features}, 2., we have:
\begin{align}
    \bigg\Vert (\mM^*)^T\phi_{n,h} \bigg\Vert_2 = \bigg\Vert \E \left[K_\psi^{-1}\psi_{n,h} | \gF_{n,h} \right] \bigg\Vert_2 \leq \E \left[ \Vert K_\psi^{-1}\psi_{n,h}\Vert_2 | \gF_{n,h} \right]\leq C_{\psi}' \;.
\end{align}
As a result we have $\E \left[ \eta_{n,h}=0 \right]$ and $\Vert \eta_{n,h}\Vert_2^2 \leq 2C_\psi'$, thus our noise is $2C_\psi'$ subgaussian. 
Using the identity $\sum_{n' < n, h \leq H} \phi_{n,h} \phi_{n,h}^T = (\mV_n^{\lambda}) - \lambda \mI$ on the terms involving $\mM^*$ and $\mW$ we can write:
\begin{align}
    \mM_n - \mM^* & = (\mV_n^{\lambda})^{-1} \sum_{n' < n, h \leq H}\phi_{n,h}\eta_{n,h} + \lambda (\mV_n^{\lambda})^{-1}(\mW-\mM^*) 
    \label{eq:biased_difference_M_equation}
\end{align}

Using Cauchy-Schwarz we have for any $X \in R^{d,d'}$:
\begin{align}
    \langle X,\mM_n - \mM^* \rangle_F & =  
         \bigg\langle X, (\mV_n^{\lambda})^{-1} \sum_{n' < n, h \leq H}\phi_{n,h}\eta_{n,h}\rangle_F + \lambda \langle X,(\mV_n^{\lambda})^{-1} (\mW-\mM^*) \bigg\rangle_F  \\
    & =  \bigg\langle (\mV_n^{\lambda})^{-1/2}X, (\mV_n^{\lambda})^{-1/2} \sum_{n' < n, h \leq H}\phi_{n,h}\eta_{n,h}\rangle_F + \lambda \langle (\mV_n^{\lambda})^{-1/2}X,(\mV_n^{\lambda})^{-1/2} (\mW-\mM^*) \bigg\rangle_F  \\
        & \leq  \bigg\Vert (\mV_n^{\lambda})^{-1/2}X\bigg\Vert_F \left( \bigg\Vert  (\mV_n^{\lambda})^{-1/2}\sum_{n' < n, h \leq H}\phi_{n,h}\eta_{n,h}\bigg\Vert_F + \lambda \Vert (\mV_n^{\lambda})^{-1/2}(\mW-\mM^*) \Vert_F\right)\;.
\end{align}
Inserting the choice $X = (\mV_n^{\lambda}) (\mM_n - \mM^*)$, using the symmetry of $(\mV_n^{\lambda})$ to split it and observing $\Vert (\mV_n^{\lambda})^{-1/2}A \Vert_{F}^2 \leq 1/\lambda_{min}(\mV_n^{\lambda}) \Vert A \Vert_F^2 \leq 1/\lambda \Vert A \Vert_F^2$ yields:

\begin{align}
    \Vert (\mV_n^{\lambda})^{1/2} (\mM_n - \mM^*)\Vert_F^2 & \leq
          \Vert (\mV_n^{\lambda})^{1/2} (\mM_n - \mM^*)\Vert_F \left( \bigg\Vert  (\mV_n^{\lambda})^{-1/2}\sum_{n' < n, h \leq H}\phi_{n,h}\eta_{n,h}\bigg\Vert_F + \sqrt{\lambda} \Vert (\mW-\mM^*) \Vert_F\right)\;.
\end{align}
Division by $\Vert (\mV_n^{\lambda})^{1/2} (\mM_n - \mM^*)\Vert_F$ yields:

\begin{align}
    \Vert \mM_n - \mM^* \Vert_{F} & \leq
           \left( \bigg\Vert  (\mV_n^{\lambda})^{-1/2}\sum_{n' < n, h \leq H}\phi_{n,h}\eta_{n,h}\bigg\Vert_F + \sqrt{\lambda} \Vert (\mW-\mM^*) \Vert_F\right)\;.
\end{align}

We can write the Frobenius norm of a matrix as the root of the squared two norms of the columns of the matrix, thus $\Vert A \Vert_F = \sqrt{\sum_j} \|(A)_j\|_2^2$. Note that each $\phi_{n,h}(\eta_{n,h})_j$ term is a vector. 

\begin{align}
    \Vert \mM_n - \mM^* \Vert_{F} & \leq
           \left( \sqrt{\sum_{j=1}^{d'}\bigg\Vert  \sum_{n' < n, h \leq H}\phi_{n,h}(\eta_{n,h})_j\bigg\Vert_{(\mV_n^{\lambda})^{-1}}^2} + \sqrt{\lambda} \Vert (\mW-\mM^*) \Vert_F\right)\;.
\end{align}

We apply the  self-normalising bound for vector-values martingales from theorem 1 in ~\citep{AbbasiYadkori2011ImprovedAF}/ theorem 20.4 ~\citep{lattimore_szepesvari_2020} to the noise term and obtain:
\begin{align}
    \Vert \mM_n - \mM^* \Vert_{F} & \leq \left(  \sqrt{d'2(C_\psi')^2\log\frac{\det((\mV_n^{\lambda}))^{1/2}}{\delta \det(\mV_0)^{1/2}}} + \sqrt{\lambda} \Vert (\mW-\mM^*) \Vert_F\right) \;.
    \label{eq:frobenius_noise_ball_determinant}
\end{align}
Invoking the elliptical potential (lemma \ref{app:lemma:eppliptical_potential_bandit}, equation \ref{eq:our_elliptical_potential_lemma}) on this yields:
\begin{align}
    \Vert \mM_n - \mM^* \Vert_F & \leq   C_\psi' \sqrt{2d'\log\left(\frac{1}{\delta}\right) + d'd \log\left( 1 + \frac{nHC_\phi}{d\lambda}\right)} + \sqrt{\lambda} \Vert \mW-\mM^*) \Vert_F\;.
    \label{eq:frobenius_noise_ball_delta}
\end{align}

Here we used the assumption $\Vert\Phi(s,a) \Vert_2^2 \leq C_\phi$. Using the $dC_\phi$ constraint from \cite{yang2019reinforcement} would have given an extra in the nominator which would have cancelled with the denominator.

\end{proof}

\subsection{Doing RL using the estimated matrix ball}
\label{subsec:app:uc_matrixrl_regret}
We start by showing that the estimation error is along the directions of the exploration. 
\begin{lemma}[similar to lemma 5 in \cite{yang2019reinforcement}]
For any $M \in B_n$ we have 
\begin{align}
    \Vert \phi(s,a)^T(M-\mM_n)\Vert_1 \leq \beta_n\sqrt{ \phi(s,a)^T((\mV_n^{\lambda}))^{-1} \phi(s,a)}
\end{align}
\label{lemma:estimation_error_in_exploration_direction}
\end{lemma}
\begin{proof}
\begin{align}
    \Vert \phi(s,a)^T (M - \mM_n)\Vert_1 
         & = \Vert \phi(s,a)^T (\mV_n^{\lambda})^{-1/2}(\mV_n^{\lambda})^{1/2}(M - \mM_n)\Vert_1\\
    & \leq \Vert\phi(s,a)^T (\mV_n^{\lambda})^{-1/2} \Vert_2 \Vert (\mV_n^{\lambda})^{1/2}(M - \mM_n)\Vert_{2,1} \\
        & = \Vert\phi(s,a)^T (\mV_n^{\lambda})^{-1/2} \Vert_2 \Vert (M - \mM_n)\Vert_{(\mV_n^{\lambda}),1}\\
    & \leq \beta_n \Vert\phi(s,a)^T (\mV_n^{\lambda})^{-1/2} \Vert_2
\end{align}
\end{proof}

While \cite{AbbasiYadkori2011ImprovedAF} directly decompose the regret, in case of MDPs with linear transition modules  we show first that the value iteration error per step is not too large which allows us to subsequently bound the regret. For notational purposes it wil be convinient to define: \begin{align}
    w_{n,h} = \sqrt{\phi(s_{n,h}, a_{n,h})^T((\mV_n^{\lambda}))^{-1}\phi(s_{n,h}, a_{n,h})} = \sqrt{\phi_{n,h}^T((\mV_n^{\lambda}))^{-1}\phi_{n,h}}   \;.
\end{align}

\begin{lemma}
[Like Lemma 6 in \cite{yang2019reinforcement}]

\begin{align}
    Q_{n,h}(s_{n,h}, a_{n,h}) - \left[r(s_{n,h}, a_{n,h}) + P(\cdot \vert s_{n,h}, a_{n,h}) V_{n, h+1} \right] \leq 2 C_\psi H \beta_n w_{n,h}
\end{align}
\end{lemma}
\begin{proof}
Similar to \cite{yang2019reinforcement} and using the updated bound in lemma \ref{lemma:estimation_error_in_exploration_direction}
\end{proof}

Using this we refine lemma 7 of \cite{yang2019reinforcement} as follows: 
\begin{lemma}[lemma 7 in \cite{yang2019reinforcement}] 
Assumptions hold, $1 \leq \beta_1 \leq \dots \leq \beta_N$, then we can bound the regret at terminal time $T = NH$:
\begin{align}
    Regret(T) \leq 2 C_\psi H\beta_n \E \left[ \sum_{n=1}^N\sum_{h=1}^H
\sqrt{\min(1,w_{n,h}^2)} \right] + \sum_{n=1}^N H\sP\left[E_n = 0 \right]
\end{align}
\label{lemma:biased_ucmatrixrl_regret_bound}
\end{lemma}

The only term that remains to be bounded is:
\begin{align}
    \sum_{n=1}^N\sum_{h=1}^H
\sqrt{\min(1,w_{n,h}^2) } \leq \sqrt{HN \sum_{n=1}^N\sum_{h=1}^H
\min(1,w_{n,h}^2)}
\end{align}
Looking at the structure of $\sum_{n=1}^N\sum_{h=1}^H
\min(1,w_{n,h}^2) $ it is tempting to reapply the classic elliptical potential lemma \ref{eq:elliptical_potential_lemma_eq}. However, this uses at step $n,h$ the $V_{n,h}$ induced Malahonobis norm. We shall however have a fixed Malahonobis norm $\Vert. \Vert_{V_{n}}$ throughout episode $n$. We derive a related log determinant lemma, which is quite similar to lemma \ref{app:lemma:eppliptical_potential_bandit}, however we need to pay a factor $H$.

\begin{lemma}[Analogous to lemma 8 in  ~\citep{yang2019reinforcement}]
\begin{align}
    \sum_{n=1}^N\sum_{h=1}^H
\min(1,w_{n,h}^2) 
\leq 2H \ln \frac{\det(V_{N+1})}{\det(V_0)}
\leq 2Hd\ln \frac{NHC_\phi + \lambda}{\lambda }
\label{eq:elliptical_potential_lemma_eq}
\end{align}
\label{lemma:log_det_stale_features}
\end{lemma}
\begin{proof}
Note that for any $u \geq 0$ it holds $\min(1,u) \leq 2\ln(1+u)$. Thus we get for the l.h.s. of \ref{eq:elliptical_potential_lemma_eq}:
\begin{align}
    \sum_{n=1}^N\sum_{h=1}^H \min(1,w_{n,h}^2)
        \leq 2 \sum_{n=1}^N\sum_{h=1}^H \ln(1 + w_{n,h}^2) \;.
\end{align}
We are now going to bound the r.h.s. by a log determinant ratio. Recalling the structure of $(\mV_n^{\lambda})$ we have:
\begin{align}
    V_{n+1} = (\mV_n^{\lambda}) + \sum_{h=1}^H\phi_{n,h}\phi_{n,h}^T \;.
\end{align}
The determinant is a multiplicative map, thus we can decompose $\det(V_{n+1})$ as:
\begin{align}
    \det(V_{n+1}) = \det(V_{n}) \times \det(I + (\mV_n^{\lambda})^{-1/2}\sum_{h=1}^H\phi_{n,h}\phi_{n,h}^T(\mV_n^{\lambda})^{-1/2}) \;.
\end{align}
Thus we get:
\begin{align}
    \det(V_{n+1}) & \geq \det((\mV_n^{\lambda}))\prod_{h=1}^H
(1 + w_{n,h}^2)^{1/H}  \geq \det(V_{n-1}) \dots \\
& \geq \det(V_0) \prod_{n=1}^N \prod_{h=1}^H
(1 + w_{n,h}^2)^{1/H}\;.
\end{align} 

Putting things together and recalling eq. \ref{eq:our_elliptical_potential_lemma} we get:
\begin{align}
    \sum_{n=1}^N\sum_{h=1}^H \ln(1 + w_{n,h}^2) 
    \leq 2H \ln \frac{\det(V_{N+1})}{\det(V_0)}  \leq 2Hd\ln \left(1 + \frac{NHC_\phi}{\lambda d} \right) \;.
\end{align}
\end{proof}

It remains to combine this with the confidence ellipsoid we constructed in theorem \ref{thm:frobenius_oracle_ellipsoid}. We begin by restating the regret bound. 

\regretBucMatrixRL*

\begin{proof}
By theorem \ref{thm:frobenius_oracle_ellipsoid} we choose:
\begin{align}
    \beta_n = C_\psi' \sqrt{2d'\log\left(\frac{1}{\delta}\right) + d'd \log\left(D\right)} + \sqrt{\lambda} \Vert W-\mM^* \Vert_F\;.
\end{align}

Then the bad event, thus the optimal laying being outside the confidence ellipsoid, occurs only with small probability:
\begin{align}
    \sP\left[\forall n \leq N | E_n = 1 \right] \geq 1 - \delta
\end{align}
Thus we get by lemma \ref{lemma:biased_ucmatrixrl_regret_bound} when using lemma \ref{lemma:log_det_stale_features} :

\begin{align}
    Regret(T, \gM) &\leq 2 C_\psi H \beta_N \sqrt{2T Hd \ln \left(1 + \frac{NHC_\phi}{\lambda d} \right)}
\end{align}

Using lemma \ref{lemma:upper_bound_doubeling} instead of lemma \ref{lemma:log_det_stale_features}  yields:
\begin{align}
    Regret(T, \gM) &\leq 2 C_\psi H \beta_N \sqrt{C_{\phi, \lambda}T d \ln \left(1 + \frac{NHC_\phi}{\lambda d} \right)}
\end{align}
In the remainder of the paper we use the regret bound obtained by using lemma \ref{lemma:upper_bound_doubeling}. 
\end{proof}

\subsection{Meta Transfer Regret}
\label{subsec:app:meta_transfer_regret}

\MetaRegretBucMatrixRL*

\begin{proof}
The first inequality is a result of applying theorem \ref{thm:regret_biased_matrix_rl_self_normalising_frobenius} with task distribution $\gT$ and recalling the definition of transfer regret in equation \ref{eq:transfer_regret}. Jensen's inequality yields the second inequality. 
\end{proof}
\section{Proof of the transfer regret for the low bias estimator}
\label{sec:app:low_bias_transfer_regret_proof}

This analysis is inspired form the analysis of the low bias case in ~\citep{pmlr-v119-cella20a}. 
We start by stating the following vectorial version of Bennett's inequality which we shall need to bound the estimation error.

\begin{lemma}[Vectorial Version of Bennett's inequality; lemma 2 ~\citep{smale2007learning}/lemma 3 ~\citep{pmlr-v119-cella20a}]
Let $\vm_1, \dots, \vm_G$ be $G$ independent random vectors in $\R^{d}$ drawn from a joint distribution $\gT$ with mean $\Bar{n}$ and variance $\sigma_m$. Assume a bounded 2-norm at all stages: $\Vert \vm_g \Vert_2 \leq C_M \; \forall g \in \left[G \right]$. For any $\delta \in (0,1)$ it holds with confidence $1-\delta$:
\begin{align}
     \frac{1}{G} \sum_{g=1}^G \Vert\vm_g - \Bar{\vm}\Vert_2 \leq \frac{2 \log(2/\delta)C_M}{G} + \sqrt{\frac{2\log(2/\delta)\sigma_M}{G}}
\end{align}

\label{eq:vectorial_version_bennetts_inequlaity}
\end{lemma}

Using the interpretation of the Frobenius norm of a matrix as the 2-norm of a flattened version of the matrix we obtain the following corollary:

\begin{corollary}[Frobenius Version of Bennett's inequality]
Let $\mM_1, \dots, \mM_G$ be $G$ independent random matrices in $\R^{d \times d'}$ sampling from a joint distribution $\gT$ with mean $\Bar{\mM}$ and variance $\sigma_M$. Assume a bounded Frobenius norm at all stages: $\Vert \mM_g \Vert_F \leq C_M \; \forall g \in \left[G \right]$. For any $\delta \in (0,1)$ it holds with confidence $1-\delta$:
\begin{align}
    \frac{1}{G} \sum_{g=1}^G| \mM_g - \Bar{\mM}|_F 
        \leq \frac{2 \log(2/\delta)C_M}{G} + \sqrt{\frac{2\log(2/\delta)\sigma_M}{G}}
\end{align}
\label{corolloary:frobenius_version_bennetts_inequlaity}
\end{corollary}

Recalling the definition of the empirical mean task $\Bar{\mW}_{G,t} = \frac{1}{NT + t} \left(\sum_{g=1}^G T M_g + tM_{G+1} \right)$ and the estimation error relative to the true mean transition core $\Bar{\mM}$ as $H_{\RM}(G+1,\Bar{\mM}) = \Vert \Bar{\mM} - \Bar{\mW}_{G,t} \Vert_F$ we get from corollary \ref{corolloary:frobenius_version_bennetts_inequlaity}
\begin{align}
    \lim_{G \rightarrow \infty}H_{\RM}(G+1,\Bar{\mM}) = 0\;.
\end{align}
We deduce that the observation error is dominated by the variance term $\max_{g \in \left[G \right]} \frac{\beta_{g,T}(1/{NH})}{\lambda_{min}^{1/2} (V^\lambda_{g,T})}$. 
From standard linear regression results \cite{lai1982least} we know that $\lambda_{min}(V_{g,T}) \geq \log T$. By construction of the $V_j^\lambda$ we have as such $\lambda_{min}(V_{g,T}^\lambda) \geq \lambda + \lambda_{min}(V_{g,T})  \geq \lambda + \log T$. We see that the smaller $T$ the larger the impact of the size of $\lambda$. Recall further our regularisation strength schedule: $\lambda = \frac{1}{TVar_{\hat{W}}}$. We see that if the variance of our estimator is large, the resulting $\lambda$ is small. This leads to smaller $\lambda_{min}(V_{g,T})$, which yields to a higher variance, thus we are left with a circle of potentially self reinforcing variance. This is particularly problematic for task distributions $\RM$ of high variance.

\LowBiasMetaRegretBucMatrixRL*
\begin{proof}
 Recall the upper bound on the meta transfer regret from theorem: \ref{thm:transfer_regret_frobenius_oracle}:
 \begin{align}
      MTR_T(\gT) \leq C C_\psi H C_\psi'd\sqrt{d'TC_{\phi, \lambda}\log\left( TD\right)\ln \left(D \right)} + CC_\psi H \sqrt{\textup{Var}_W\lambda T C_{\phi, \lambda}d \ln \left(D \right)}\;. 
 \end{align}
Inserting $\lambda = \frac{1}{TVar_{\mW}}$ and focussing for the ease of presentation on the dominant term we get:
\begin{align}
    MTR_T(\gT) & \leq C C_\psi HdC_\psi' \sqrt{C_{\phi, \lambda}d'T \log\left( T + \frac{T^3C_\phi Var_{\mW}
    }{d}\right)  }
\end{align}
 
Recall the choice of bias matrix estimator from equation \ref{eq:bias_estimator_weighted_average}
\begin{align*}
    \hat{\mW}_{G,n,h}  =  \sum_{g=1}^{G-1} \frac{T}{Z}\hat{\mM}_{g,T}  + \frac{nH+h}{Z} \hat{\mM}_{G,n,h}. 
\end{align*}

The regret upper bound for this choice of bias matrix amounts to:
\begin{align}
    MTR_T(\gT) & \leq C C_\psi HdC_\psi' \sqrt{C_{\phi, \lambda}d'T \log\left( T + \frac{T^3C_\phi Var_{\hat{\mW}_{G,n,h}}
    }{d}\right)  }
    \label{eq:low_bias_estimator_dominant_term_mtr}
\end{align}

 By the triangle inequality we have:
\begin{equation*}
    \sqrt{Var_{\hat{\mW}_{G,n,h}}} = \sqrt{\mathbb{E}_{\mM \sim \gT }\Bigg[\Vert{\mM - \hat{\mW}_{G,n,h}} \Vert_F\Bigg]} \leq \sqrt{Var_{\Bar{\mM}}} + \sqrt{\epsilon_{G,n,h}(\gT)}\;.
\end{equation*}
Plugging this into equation \ref{eq:low_bias_estimator_dominant_term_mtr} yields the desired meta transfer regret bound. It remains to analyse the estimation error $\epsilon_{G,T}(\RM)$:
    \begin{align*}
    \sqrt{\epsilon_{G,T}(\RM)} &= \Vert{\bar{\mM} - \hat{\mW}^\lambda_{G,T}}_F \Vert\\
    &    \leq \Vert{\bar{\mM} - \bar{\mW}_{G, T}}\Vert_F + \Vert{\bar{\mW}_{N,T} - \hat{\mW}^\lambda_{G,T}}\Vert_F\\
    &= H_{\RM}(G+1, \bar{\mM}) + \Vert{\bar{\mW}_{N,T} - \hat{\mW}^\lambda_{G,T}}\Vert_F\\
        &\leq H_{\RM}(G+1, \bar{\mM}) + \max_{g \in \left[ G\right]}\Vert{\mM_g - \hat{\mM}_{g,T}}\Vert_F\\
    &\leq H_{\RM}(G+1, \bar{\mM}) + \max_{g \in \left[ G\right]}\frac{\Vert{\left(\mM_g - \hat{\mM}_{g,T}\right) (\mV^\lambda_{g,T})^{1/2}}\Vert_F}{\lambda^{1/2}_{\min}(\mathbf{V}^\lambda_{g,T})}\\
        &\leq H_{\RM}(G+1, \bar{\mM}) + \max_{g \in \left[ G\right]}\frac{\beta^\lambda_j\big(1/T\big)}{\lambda^{1/2}_{\min}(\mathbf{V}^\lambda_{g,T})}. 
\end{align*}
\end{proof}

\section{Proof of the transfer regret for the high bias estimator}
\label{sec:app:high_bias_transfer_regret_proof}
The estimator as well as the way to prove its properties is adapted from section 5 in ~\citep{pmlr-v119-cella20a} to the matrix case. 
To this end we introduce firstly an  additional variable:
\begin{equation*}
    \Bar{\mW}_{G,t+1}' = \left(\Tilde{\mV}_{G,t} \right)^{-1} \Bigg(\sum_{g=1}^{G}\mV_{g,T} \mM_g + \mV_{N+1,t} \mM_{G+1} \Bigg) \;.
\end{equation*}

Using the triangle inequality we can now separate error causes. We get the \emph{estimation error} $\hat{\mW}^\lambda_{G,t+1} - \Bar{\mW}_{G,t+1}'$ which we handle in \ref{lemma:EstimationError_high_bias} and the \emph{estimation bias} $\Bar{\mW}_{G,t+1}' - \Bar{\mW}_{G,t+1}$ which we cover in lemma \ref{lemma:EstimationBias_high_bias}.

\begin{lemma}[Estimation error of global ridge regression]

\label{lemma:EstimationError_high_bias}
    The following rewriting holds:
    \begin{equation*}
        \hat{\mW}^\lambda_{G,t+1} - \Bar{\mW}_{G,t+1}' = \left(\Tilde{\mV}^\lambda_{G,t}\right)^{-1} \Bigg( \sum_{g=1}^{G}\sum_{s=1}^{T} \phi_{g,s} \eta_{g,s} + \sum_{s=1}^t \phi_{G+1,s} \eta_{G+1,s} \Bigg) - \lambda \left(\Tilde{\mV}^\lambda_{G,t}\right)^{-1}\Bar{\mW}_{G,t+1}'
    \end{equation*}
\end{lemma}
\begin{proof}
Follows immediately from \citep{pmlr-v119-cella20a}.
\end{proof}

\begin{lemma}[Estimation bias of global ridge regression]\label{lemma:EstimationBias_high_bias}
 From Section \ref{sec:app:low_bias_transfer_regret_proof}, we use:
\begin{align*}
    \Bar{\mW}_{G,t+1} = \frac{1}{GT+t} \left( \sum_{g=1}^{G} T \mM_g + t \mM_{G+1} \right).
\end{align*}
Differently from $\Bar{\mW}_{G,t}'$ this definition is a weighted average of the transition cores of the $G$ previously encountered tasks. Thus, we have:
\begin{align*}
    \Vert \Bar{\mM} - \Bar{\mW}_{G,t}'\Vert_F
        &\leq \frac{1}{NT+t} \sum_{g=1}^{G}\bigg[ \Vert \Bar{\mM} - \Bar{\mW}_{G,t} \Vert_F + (GT+t) \Vert \Bar{\mW}_{G,t} - \Bar{\mW}_{G,t}' \Vert_F\bigg]\\
    &= H_{\gT}(G + 1,\Bar{\mM}) + \Vert\Bar{\mW}_{G,t} - \Bar{\mW}_{G,t}' \Vert_F
\end{align*}
where we have denoted with $H_\gT(G + 1,\Bar{\mM})$ according to what we have done in subsection~\ref{subsection:buc-matrixrl-low_bias_estimator}. We can now focus on the term $\Vert\Bar{\mW}_{G,t}' - \Bar{\mW}_{G,t} \Vert_F$ which can be equivalently rewritten as:
\begin{align*}
         \Vert\Bar{\mW}_{G,t+1}' - \Bar{\mW}_{G,t+1}\Vert_F &= \bigg\Vert\left(\Tilde{\mV}_{G,t}\right)^{-1} \sum_{g=1}^{G} \left( \mV_{g,T} \mM_g + \mV_{G+1,t} \mM_{G+1} \right) - \Bar{\mW}_{G,t}\bigg\Vert_F\\
    &\leq \sum_{g=1}^{G} \Big| \Tilde{\mV}_{G,t}^{-1} \mV_{g,T} \Big| \Vert \mM_g - \Bar{\mW}_{G,t}\Vert_F + \Big| \Tilde{\mV}_{G,t}^{-1} \mV_{G+1,t} \Big| \Vert \mM_{G+1} - \Bar{\mW}_{G,t}\Vert_F\\
        &\leq \sum_{g=1}^{G} H_\gT(G+1,\mM_g) \Big| \Tilde{\mV}_{G,t}^{-1} \mV_{g,T} \Big| + H_\gT(G+1,\mM) \Big|\Tilde{\mV}_{G,t}^{-1} \mV_t \Big|\\
    &= \sum_{g=1}^{G} H_\gT(G+1,\mM_g) \sigma_{\max} \bigg(\mV_{g,t}\Tilde{\mV}^{-1}_{G,t} \bigg) + H_\gT(G+1,\mM_{G+1}) \sigma_{\max} \bigg(\mV_{g,t}\Tilde{\mV}^{-1}_{G,t} \bigg)\\
        &\leq (G+1) \max_{g=1,\dots,G+1} \Bigg( H_\gT(G+1, \mM_j) \sigma_{\max} \bigg( \mV_{g,t}\Tilde{\mV}^{-1}_{G,t} \bigg) \Bigg) \\
    &= (G+1) \max_{g=1,\dots,G+1} \Tilde{H}(G+1, \mM_g)
\end{align*}
We have used the fact that the matrix norm of a given matrix $A$ induced by the Euclidean norm corresponds to the spectral norm, which is the largest singular value of the matrix $\sigma_{\max}(A)$ .

\end{lemma}

\HighBiasMetaRegretBucMatrixRL*
\begin{proof}

The bound on the meta transfer regret is obtained similar to the first part of the proof of theorem \ref{thm:transfer_regret_low_bias_bias_estimator} by using the triangle inequality:
\begin{equation*}
    \sqrt{Var_{\hat{\mW}_{G,n,h}}} =\sqrt{\E_{\mM \sim \gT }\Bigg[\Vert\mM - \hat{\mW}_{G,T}^\lambda \Vert^2_F\Bigg]} \leq \sqrt{Var_{\Bar{\mM}}} + \sqrt{\epsilon_{G,t}(\gT)} \;.
\end{equation*}
It thus remains to bound the estimation error. According to lemma \ref{lemma:EstimationBias_high_bias} we rewrite:

\begin{equation*}
    \sqrt{\epsilon_{G,t}(\gT)} \leq H_\gT(G + 1,\Bar{\mM}) +  (G+1) \max_{g=1,\dots,G+1} \Tilde{H}(G+1, j) + \Vert\Bar{\mW}_{G,T}' - \Hat{\mW}_{G,T}^\lambda\Vert_F
\end{equation*}
Recalling the definition of the Frobenius Malahonobis norm for matrices (equation \ref{eq:matrix_malahonis_frobenius}) it remains only to apply lemma \ref{lemma:EstimationError_high_bias} which gives:
\begin{align*}
    \Vert \Bar{\mW}_{G,T}' - \Hat{\mW}_{G,T}^\lambda \Vert_F
        &=  \bigg\Vert\left(\Tilde{\mV}^\lambda_
        {G,T}\right)^{-1} \left( \sum_{g=1}^{G}\sum_{s=1}^{T} \phi_{g,s} \eta_{g,s} + \sum_{s=1}^T \phi_s \eta_s \right)\bigg\Vert_F+ \Vert\lambda \left(\Tilde{\mV}^\lambda_{G,T}\right)^{-1}\Bar{\mW}_{G,T}' \Vert_F \\
    &\leq \bigg\Vert\sum_{g=1}^{G}\sum_{s=1}^{T} \phi_{g,s} \eta_{g,s} + \sum_{s=1}^T \phi_s \eta_s\bigg\Vert_{\left(\Tilde{\mV}^\lambda_{G,T}\right)^{-2}} + \lambda \Vert\Bar{\mW}_{G,T}'\Vert_{_{\left(\Tilde{\mV}^\lambda_{G,T}\right)^{-2}}}\\
        &\leq \frac{1}{\lambda^\frac{1}{2}_{\min}(\Tilde{\mV}^\lambda_{G,T})}\bigg\Vert\sum_{g=1}^{G}\sum_{s=1}^{T} \phi_{g,s} \eta_{g,s} + \sum_{s=1}^T \phi_s \eta_s\bigg\Vert_{\left(\Tilde{\mV}^\lambda_{G,T}\right)^{-1}} + \frac{1}{\lambda_{\min}(\Tilde{\mV}^\lambda_{G,T})} \Vert\Bar{\mW}_{G,T}'\Vert_{F}\\
    &\leq \frac{1}{\lambda^\frac{1}{2}_{\min}(\Tilde{\mV}^\lambda_{G,T})} C_\psi'\sqrt{2\log\bigg(T+\frac{(GT + T)TL^2}{\lambda d}\bigg)} + \Vert\Bar{\mW}_{G,T}' - \Bar{\mW}_{G,T}\Vert_{F} + \frac{1}{\lambda_{\min}(\Tilde{\mV}^\lambda_{G,T})} \Vert\Bar{\mW}_{G,T}\Vert_F\\
        &\leq \frac{1}{\lambda^\frac{1}{2}_{\min}(\Tilde{\mV}^\lambda_{G,T})} C_\psi'\sqrt{2\log\bigg(T+\frac{(GT + T)TL^2}{\lambda d}\bigg)} + \Vert\Bar{\mW}_{G,T}' - \Bar{\mW}_{G,T}\Vert_{F} + \frac{S}{\lambda_{\min}(\Tilde{\mV}^\lambda_{G,T})}\\
    &\leq \frac{1}{\lambda^\frac{1}{2}_{\min}(\Tilde{\mV}^\lambda_{G,T})} C_\psi'\sqrt{2\log\bigg(T+\frac{(GT + T)TL^2}{\lambda d}\bigg)} + (G+1) \max_{g=1,\dots,G+1} \Tilde{H}(G+1, j) \\ & \quad + \frac{S}{\lambda_{\min}(\Tilde{\mV}^\lambda_{G,T})}\\
\end{align*}
In the last inequality we used again lemma \ref{lemma:EstimationBias_high_bias}.
We can now introduce $\nu_{\min} = \lambda_{\min}\big(\Tilde{\mV}_{G,T}\big)$ as eigenvalue of the global feature matrix. It follows as desired:

\begin{align}
    \sqrt{\epsilon_{G,T}(\RM)} & \leq H_{\RM}(G+1, \Bar{M})  + 2(G+1)\max_{g \in \left[G+1 \right]} \Tilde{H}(G+1, M_g)\\
    & \quad \quad + \frac{d C_M}{\lambda + \nu_{\min}} + C_\psi' \sqrt{ \frac{2}{\lambda + \nu_{\min}}\log\left(GH + \frac{GN^2H^2C_\phi}{\lambda d}\right)} \;.
\end{align}

\end{proof}

\end{document}